\documentclass[runningheads]{llncs}
\usepackage{graphicx}
\usepackage[utf8]{inputenc}

\usepackage[hidelinks]{hyperref}

\usepackage[T1]{fontenc}
\usepackage[USenglish]{babel}
\usepackage{amssymb,amsmath}
\usepackage{booktabs}
\usepackage{mathtools}
\usepackage{listings}
\usepackage[noabbrev]{cleveref} 
\usepackage{enumitem} 
\usepackage{cite}
\usepackage[algoruled,vlined,linesnumbered]{algorithm2e}

\usepackage{fca}

\newcommand{\dR}{\mathbb{R}}
\newcommand{\dN}{\mathbb{N}}
\newcommand{\dK}{\mathbb{K}}
\DeclareMathOperator{\softmax}{softmax}
\DeclareMathOperator{\NN}{NN}
\DeclareMathOperator{\soft}{softmax}
\DeclareMathOperator{\SG}{SG}
\DeclareMathOperator{\CBoW}{CBoW}
\DeclareMathOperator{\TC}{T}
\DeclareMathOperator{\TP}{T}
\DeclareMathOperator{\ext}{ext}
\newcommand\logeq{\mathrel{\vcentcolon\Leftrightarrow}}

\makeatletter
\def\blfootnote{\xdef\@thefnmark{}\@footnotetext}
\makeatother

\let\cref\Cref
\let\subset\subseteq
\let\oldexists\exists
\renewcommand{\exists}{\oldexists\mkern2mu}
\let\oldforall\forall
\renewcommand{\forall}{\oldforall\mkern2mu}

\begin{document}
\title{FCA2VEC: Embedding Techniques for Formal Concept Analysis}
%
%
\author{Dominik Dürrschnabel\inst{1,2} \and Tom Hanika\inst{1,2} \and
  Maximilian Stubbemann\inst{1,2,3}}
%
%
\institute{Knowledge \& Data Engineering Group,
  University of Kassel, Kassel,  Germany\\
  \and
  Interdisciplinary Research Center for Information System Design\\
  University of Kassel, Kassel, Germany\\
  \and L3S Research Center, Leibniz University Hannover, Hannover, Germany
  \email{duerrschnabel@cs.uni-kassel.de, hanika@cs.uni-kassel.de,
    stubbemann@l3s.de} }
\maketitle 
\blfootnote{Authors are given in alphabetical order.  No priority in authorship
  is implied.}

\begin{abstract}
  Embedding large and high dimensional data into low dimensional vector spaces
  is a necessary task to computationally cope with contemporary data sets.
  Superseding ‘latent semantic analysis’ recent approaches like ‘word2vec’ or
  ‘node2vec’ are well established tools in this realm. In the present paper we
  add to this line of research by introducing  ‘fca2vec’, a family of embedding
  techniques for formal concept analysis (FCA). Our investigation contributes to
  two distinct lines of research. First, we enable the application of FCA
  notions to large data sets. In particular, we demonstrate how the cover
  relation of a concept lattice can be retrieved from a computational feasible
  embedding. Secondly, we show an enhancement for the classical node2vec
  approach in low dimension. For both directions the overall constraint of FCA
  of explainable results is preserved.  We evaluate our novel procedures by
  computing fca2vec on different data sets like, wiki44 (a dense part of the
  Wikidata knowledge graph), the Mushroom data set and a publication network
  derived from the FCA community.
  
  \keywords{Vector~Space~Embedding \and Covering~Relation \and Link Prediction \and
    Word2Vec \and Complex~Data \and Formal~Concept~Analysis \and Closed~Sets
    \and Low~Dimensional~Embedding}
\end{abstract}

\section{Introduction}
\label{sec:introduction}
A common approach for the study of complex data sets is to embed them into
appropriate sized real-valued vector spaces, e.g., $\mathbb{R}^{d}$, where $d$ is
a small natural number with respect to the dimension of the original data. This
enables the application of the well understood and extensive tool set from
linear algebra. The practice is propelled by the presumption that relational and
other features from the data will be translated to positions and distances in
$\mathbb{R}^{d}$, at least up to some extent. Especially relative distances of
embedded entities are often meaningful, as shown in seminal works
like~\cite{journals/corr/abs-1301-3781}. For example,
in~\cite{DBLP:conf/aaai/WangZFC14} the authors employed the addressed embeddings
for the complex data kind knowledge graphs (KG). Particularly the authors
from~\cite{DBLP:journals/corr/abs-1710-04099,DBLP:journals/semweb/RistoskiRNLP19}
demonstrated an embedding of the Wikidata
KG~\cite{DBLP:journals/cacm/VrandecicK14} in a 100-dimensional model.

The embedding approaches were shown to be successful for many research fields
like \emph{link prediction}, \emph{clustering}, and \emph{information
  retrieval}. Hence, today they are widely applied. Despite that they also do
exhibit multiple shortcomings. One of the most pressing is the fact that learned
embeddings elude themselves from
interpretation~\cite{DBLP:journals/corr/GoldbergL14} and explanation, even
though they are conducted in comparatively low dimension, e.g., 100. In our work
we want to overcome this disadvantage. We step in with an exploration of the
connection of formal concept analysis notions on the one side and proven
embedding methods like \emph{word2vec} or \emph{node2vec} on the other. Our
investigation is two fold and can be represented by two questions: First, how
can vector space embeddings be exploited for coping more efficiently with
problems from FCA? Secondly, to what extent can conceptual structures from FCA
contribute to the embedding of formal context like data structures, e.g.,
bipartite graphs? In order to deal with the afore mentioned shortcomings, namely
the non-interpretability/explainability, we limit our search for results to the
posed questions using an additional constraint. We want to utilize only two and
three dimensions for the to be calculated real-valued vector space
embeddings. By doing so we ensure, at least to some extent, the possibility of
some human comprehension, interpretation or even explanation of the results.

Equipped with this problem setting we perform different theoretical and
practical considerations. We revisit previous work by
Rudolph~\cite{conf/Rudolph07} and propose different models for learning closure
operators using neural networks. Conversely to this we develop a natural
procedure for improving low dimensional node embeddings using formal
concepts. Here we may point out the peculiarity that our approach does not
require the otherwise necessary parameter tuning. We evaluate the introduced
techniques by experiments on covering relations and link prediction.  Finally, we
present some first ideas on how to treat, identify and extract functional
dependencies (i.e., attribute implications and object implications) using partially
learned closure operators.

The following is divided into five sections. First, we start with an overview
over related work in~\cref{sec:rel}. This will include in particular previous
work from FCA.  In~\cref{sec:foundations} we recollect operations and notations
from formal concept analysis and word2vec. The next section contains our
modeling which connects the field of FCA with word2vec like approaches. Here we
provide some theoretical insights into what aspects of closure operators can be
learned through embeddings. This is followed by an experimental evaluation
in~\cref{sec:experiments} employing three medium sized data sets, i.e., the well
known \emph{Mushroom} context, a dense extract of the Wikidata KG,
called~\emph{wiki44k}~\cite{DBLP:conf/icfca/Hanika0S19}, and a bipartite
publication graph consisting of authors in the FCA community.\footnote{The data
  was extracted from \url{https://dblp.uni-trier.de/} and is part of the testing
  data set for the formal concept analysis software \emph{conexp-clj}, which is
  hosted at GitHub, see
  \url{https://github.com/tomhanika/conexp-clj/tree/dev/testing-data}.} We
conclude our work with~\cref{sec:conc} providing further research questions to
be investigated.

\section{Related Work}%
\label{sec:rel}
We will employ in our work learning models based on neural networks, in
particular but not limited to, word2vec~\cite{journals/corr/abs-1301-3781} and
derived works like node2vec~\cite{grover16}. To the best of our knowledge there
are no previous works on embedding (FCA) closure systems into real-valued vector
spaces using a neural network (NN) based learning setting. However, there is an
plethora of principle investigations for embedding \emph{finite ordinal data} in
real vector spaces, first of all \emph{measurement
  structures}~\cite{SCOTT1964233} (which we found via~\cite{wille1997role}). In
there the author investigated the basic feasibility of such an endeavor. Along
this line of research in the realm of FCA is~\cite{utaRepresentation}, which is
also focused on ordinal structures, in particular \emph{ordinal formal
  contexts}. The only FCA related learning model based approach we are aware of
was investigated by authors in~\cite{LSAapproach} and uses latent semantic
analysis (LSA). Their analysis demonstrates that the LSA learning procedure does
lead to useful structures. Nonetheless we refrain from considering LSA for our
work. The by us investigated NN procedures posses a crucial advantage over LSA:
they are able to cope with incremental updates of the relational data
efficiently~\cite{AAAI1714732}. In the realm of modern complex data structures,
such as Wikidata, this is a necessity.  More research in the line of such data
structures, namely Resource Description Framework Graphs, was explored by the
works~\cite{DBLP:journals/semweb/RistoskiRNLP19,
  DBLP:journals/corr/abs-1710-04099,DBLP:conf/aaai/WangZFC14}. For this the
authors use simple as well as sophisticated approaches. The overall goal in
these compositions is to provide node similarity corresponding to the
underlying relational structure. Since our goal is to excavate and employ a
hidden conceptual relation we will develop an alternative NN method for formal
context like data. For this we also foster from~\cite{conf/Rudolph07}. In there
the author conducts a more fundamental approach for employing NN in the realm of
closure systems. Notably, an encoding of closure operators through NN using FCA
is presented.

\section{Foundations}
\label{sec:foundations}

\subsubsection{Formal Concept Analysis}
\label{sec:FCA}

Before we start with our modeling, we want to recall necessary notions from
formal concept analysis. For a detailed introduction we refer the reader
to~\cite{fca-book}.  A \emph{formal context} is a triple
$\mathbb{K}\coloneqq(G,M,I)$, where $G$ represents the finite \emph{object set},
$M$ the finite \emph{attribute set}, and $I\subseteq G\times M$ a binary
relation called \emph{incidence}.  We say for $(g,m)\in I$ that object $g\in G$
has attribute $m\in M$. In this structure we find a (natural) pair of derivation
operators
$\cdot'\colon\mathcal{P}(G)\to\mathcal{P}(M),~A\mapsto A'\coloneqq \{m\in M\mid
\forall g\in A\colon (g,m)\in I\}$ and
$\cdot'\colon\mathcal{P}(M)\to\mathcal{P}(G),~ B\mapsto B'\coloneqq\{g\in G\mid
\forall m\in B\colon (g,m)\in I\}$. Those give rise to the notion of a
\emph{formal concept}, i.e., a pair $(A,B)$ consisting of an object set
$A\subseteq G$, called \emph{extent}, and an attribute set $B\subseteq M$,
called \emph{intent}, such that $A'=B$ and $B'=A$ holds. The set of all formal
concepts ($\mathfrak{B}(\mathbb{K})$) constitutes together with the order
$(A_1,B_1)\leq(A_2,B_2)\logeq A_1\subseteq A_2$ a lattice~\cite{fca-book},
called \emph{formal concept lattice} and denoted by
$\underline{\mathfrak{B}}(\mathbb{K})\coloneqq(\mathfrak{B}(\mathbb{K}),\leq)$. Throughout
this work we consider formal contexts with $\forall g\in G: \{g\}’\neq\emptyset$
and $\forall m\in M:\{m\}’\neq\emptyset$.

\subsection{Word2Vec}
\label{sec:word2vec}
We adapt the word2vec approach~\cite{journals/corr/abs-1301-3781,
  mikolov2013distributed} that generates vector embeddings for words from large
text corpora. The model gets as input a list of sentences. It is then trained
using one of two different approaches: predicting for a target word the context
words around it (the \emph{Skip-gram} model, called SG); predicting from a set
of context words a target word (the \emph{Continuous Bag of Words} model, called
CBoW). In detail, word2vec works as described in the following.

Let $V=\{v_1,\dotsc,v_n\}$ be the vocabulary. We identify $V$ as a subset of the
vectorspace $\dR^n$ via $\phi: V \to \dR^n, v_i \mapsto e^i$, the $i$-th vector
of the standard basis of $\dR^n$. This identification is commonly known under
the term \emph{one-hot encoding}. The learning task then is the following: Find
for a given $d \in \dN$ with $d\ll n$ a linear map $\varphi: \dR^n \to \dR^d$,
i.e., a matrix $W\in \dR^{d \times n}$ which obeys the goal: words that appear
in similar contexts shall be mapped closely by $\varphi$. The final embedding
vectors of the words of the vocabulary are given by the map
\begin{equation}
  \label{eq:embedding}
  \Upsilon: V \to \dR^d, v \mapsto \varphi(\phi(v)).
\end{equation}
 
To obtain such an embedding, word2vec uses a neural network
approach. This network consists of two linear maps and a softmax activation
function, cf.~\cref{fig:2vec}. The first linear function maps the input from
$\dR ^n$ to $\dR^d$, the second one from $\dR^d$ back to $\dR^n$. In
detail, the neural net function has the structure
\begin{equation}
  \label{eq:net}
  \NN: \dR^n \to \dR^n, x\mapsto \softmax(\psi(\varphi(x))),
\end{equation}
where $\varphi: \dR^n \to \dR^d$ and $\psi: \dR^d \to \dR^n$ are linear maps
with corresponding matrices $W \in \dR^{d \times n}$ and
$U \in \dR^{n \times d}$. The activation function $\softmax$ is given by
\begin{equation*}
  \label{eq:softmax}
\soft: \dR^n \to \dR^n,
\begin{pmatrix}
  x_1\\ \vdots \\ x_n
\end{pmatrix}
\mapsto \frac{1}{\sum_{l=1}^n\exp(x_l)}
\begin{pmatrix}
  \exp(x_1) \\ \vdots \\ \exp(x_n)
\end{pmatrix}.
\end{equation*}

The function $\varphi$ is then used in word2vec for creating embeddings of the
words $v \in V$ via \cref{eq:embedding}. If we use the notion of layers, as
described in~\cite{bishop06}, the neural network function is a three-layer
network, consisting of an input ($I_{L}$), hidden ($H_{L}$), and output layer
($O_{L}$). In this notation the hidden layer is used to determine the
embeddings. We refer the reader again to~\cref{fig:2vec}. In the realm of
word2vec Mikolov et.~Al.~\cite{journals/corr/abs-1301-3781} proposed two
different approaches to obtain the matrices $W$ and $U$ from input data. Those
are called the Skip-gram and the Continuous Bag of Words architecture. We
recollect them in the following.

\begin{figure}[t]
  \centering
  \hfill\includegraphics{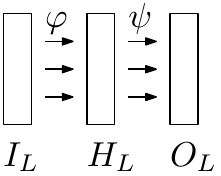}\hfill
    \includegraphics{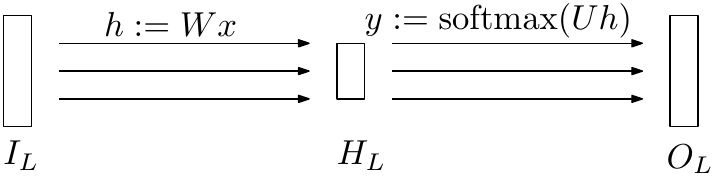}\hfill\null
    \caption{Left: A generic neural network consisting of 3 layers. Right: The
      strucutre of the word2vec architecture. The neural network consists
      of an input and output layer of size $n$ and a hidden layer of size $d$,
      where $d \ll n$.}
    \label{fig:2vec}
\end{figure}

\subsubsection{The Skip-Gram and the Continuous Bag of Words Architecture}
The SG architecture trains the network to predict for a given \emph{target word}
the \emph{context words} around it. Training examples consist of a target word
and a finite sequence of context words. We formalize these as tuples
$(t,(c_i)_{i=0}^l) \in V \times V^{< \dN}$, where $V^{< \dN}$ is the set of
finite sequences of elements of $V$. The SG architecture generates the
input-output pairs $(\phi(t),\phi(c_0)),\dotsc, (\phi(t),\phi(c_l))$ as training
examples, where $\phi$ is the one-hot encoding function, as introduced above.
In the CBoW model, in contrast to SG, the training pairs are generated
differently from $(t,(c_i)_{i=0}^l)$. We take the middle point of the list of
vectors $\phi(c_{0}),\dotsc,\phi(c_{l})$ and try to predict the target word
$\phi(t)$, hence the generated input-output training pair is
$(\frac{1}{l}\sum_{i=0}^l \phi(c_i), \phi(t))$. Both architectures employ the
same kind of a loss function to learn the weights of $W$ and $U$. The error term
is computed through cross-entropy loss. The backpropagation is done via
stochastic gradient descent. A detailed explanation can be found
in~\cite{rong14}.

In word2vec, the pairs of target word and context words are generated from text
data sequences, i.e., lists of sentences. The word2vec approach has a window
size $m \in \dN$ as parameter, i.e., for a given sentence $s=(w_i)_{i=0}^l\in
V^{< \dN}$ pairs of target word and context word sequences are defined in the
following manner. For every $i \in \{0,\dotsc ,l\}$ a reduced window size $m_{i}
\in \{0, \dotsc, m\}$ is chosen randomly and the pair $(w_i,(w_{i+k})_{k=-m_i, k
  \neq 0}^{m_i})$ is used as target word context word sequence pair. Of course,
$m_{i}$ is to be chosen reasonable with respect to $l$.

Note, in the case of word embeddings, the size of the vocabulary often reaches a
level where computing the softmax is computationally infeasible.  Hence, the
softmax layer is often approximated/replaced by one of the two following
approaches. The \emph{hierachical softmax}~\cite{mini08} stores the elements of
the vocabulary in a binary Huffmann tree and then only uses the values to the
path to an element to compute its probability.  Another approach presented
in~\cite{mikolov2013distributed} is \emph{negative sampling}, which uses the
sigmoid function. In the experimental part of this work we deal with formal
contexts of a size where applying the softmax layer is possible.

\section{Modeling}
This section is split in two parts following the two mentioned research
directions. In the first part we demonstrate how embeddings can be used in order
to retrieve (classical) FCA relevant features from data. We will discover that
different aspects of closure operators can be encoded into real-valued vector
space embeddings through neural network techniques. In particular, we are
looking at covering relations as well as canonical bases.  In the second part we
propose a straightforward approach for embedding objects and attributes with
respect to their conceptual structure. While the first part deviates from the
classical word2vec approach due to theoretical considerations, the second part
translates FCA notions to the word2vec model.  Both investigations are governed
by the overall goal from FCA to create explainable methods. To this end we apply
for all our methods only low dimensional embeddings, i.e., two or three
dimensions. Hence, these embeddings comprise the potential for human
interpretability or even explainability, in contrast to high dimensional ones.

\subsection{Retrieving FCA Features Through Closure2Vec}
\label{sec:retrieving}
The goal of this section is to employ the ideas from word2vec to improve the
computational feasibility of common tasks in the realm of Formal Concept
Analysis. Doing so, we analyze different approaches and finally settle with an
novel embedding technique that can provide more efficient computations. In
particular, we consider the FCA problems of finding the covering relation of the
concept lattice structure and the rediscovering of canonical bases. The linchpin
of our investigation is the encapsulation of the closure operator of a formal
context.

Analogue to the approach of word2vec we want to achieve a meaningful embedding
of the closure operator into $R^d$ for $d=2$ or $d=3$. For the rest of this part
we assume that for both, the attribute set and the object set are indexed, i.e.,
for some context $(G,M,I)$ we denote the object set by $G=\{g_1,g_2,\dotsc\}$
and the attribute set by $M=\{m_1,m_2,\dotsc\}$. This enables the possibility
for defining the \emph{binary encoding} of an object set $A$ as the vector $v\in
\{0,1\}^{|M|}$, with the $v_i= 1$ if and only if $m_i\in A$. Dually this can be
done for attribute sets.

\subsubsection{Exact Representation of the Closure Operator}
\label{sec:rudolph}

Thanks to a previous work by S.~Rudolph~\cite{conf/Rudolph07}, we are aware that
it is possible to represent any closure operator on a finite set into a neural
network function using formal concept analysis. The network, as proposed
in~\cite{conf/Rudolph07}, consists of an \emph{input layer} $I_{L}\coloneqq
\{0,1\}^{|M|}$, a \emph{hidden layer} $H_{L}\coloneqq\{0,1\}^{|G|}$ as well as an
\emph{output layer} called $O_{L}\coloneqq \{0,1\}^{|M|}$. The mapping between
$I_{L}$ and $H_{L}$ is defined as $\varphi = t \circ w$ consisting of a linear
mapping $w$ with transformation matrix $W=(w_{jh})\in\{-1,0\}^{|M| \times |G|}$,
such that
\[w_{jh}\coloneqq
  \begin{dcases}
    0 & \text{if } (g_j,m_h) \in I,\\
    -1 & \text{otherwise,}
  \end{dcases}\]
and a non-linear activation function $t:\mathbb{R}^{|G|}\to \mathbb{R}^{|G|}$ with each
component being mapped using the function $\tilde{t}:\mathbb{R} \to \{0,1\}$ defined as
\[\tilde{t}(x)=
\begin{dcases}
  1 & x = 0,\\
  0 & x < 0.
\end{dcases}\]

The mapping between $H_{L}$ and $O_{L}$ is defined analogously by $\psi = \hat t
\circ \hat w$, where once again $\hat w$ is a linear mapping with transformation
matrix $\hat W=W^{T}$. The function $\hat t: \mathbb{R}^{|M|} \to
\mathbb{R}^{|M|}$ is once again defined component wise with each component being
$\hat t$. Using this construction the function $\varphi \circ \psi$ encapsulates
the closure operator. To find the closure of some attribute set $B \subset M$,
one has to compute $\varphi \circ \psi$ of its binary encoding. Similar to the
both derivation operators introduced in~\cref{sec:foundations} does the mapping
$\varphi$ compute the attribute derivation and the mapping $\psi$ the object
derivation, both in their binary encoding.

\subsubsection{Considering the Unconstraint Problem}
\label{sec:why}

Considering the well established word2vec architecture the following
idea seems intuitive. Take the neural network layers as defined in Rudolph's
architecture, but replace the hidden layer by a layer containing either two or
three dimensions, i.e., we have $H_{L}=\mathbb{R}^d$. Instead of presetting
$\varphi$ and $\psi$, as in the last section, we want to retrieve them through
machine learning. However, it may be noted that it is not meaningful to allow
arbitrary functions. To see this, consider the following example with $d=1$. Let
$\mathop{s}:\{0,1\}^{|M|}\to\mathbb{N}$ be an injective mapping from the set of
binary vectors of length $|M|$ to the natural numbers. Naturally there is an
inverse map $\mathop{s}^{-1}:s[\{0,1\}^{|M|}]\to\{0,1\}^{|M|}$, where
$s[\{0,1\}^{|M|}]$ denotes the image of $\mathop{s}$ of the domain. Since
$\mathbb{N}$ is contained in $H_{L}$, we may find a natural continuation of
$\mathop{s}^{-1}$ to $\mathbb{R}$ by
$\overline{\mathop{s}^{-1}}:\mathbb{R}\to\{0,1\}^{|M|}$ such that
$\overline{\mathop{s}^{-1}}(x)\coloneqq s^{-1}(\lfloor x\rfloor)$. Furthermore,
let $\mathop{cl}$ be the double application of the derivation operator, i.e.,
$(\cdot’)’$ in the binary encoding.  Using this setup let $\varphi\coloneqq s$
and $\psi\coloneqq\overline{\mathop{s}^{-1}}\circ\mathop{cl}$.  Using these
function one can easily see that even though the neural network is able to
compute the closure operator, the layer $H_{L}$ contains no information about
the formal context. This suggests that the set of possible functions has to be
further constrained.

\subsubsection{Representing Closure Operators Using Linear Functions}
\label{sec:linear}

Rudolph’s approach for representing a closure operator by a neural network
function is sound and complete. It consists, as discussed of two linear
functions and two non-linear activation functions. The later, however, is
incompatible with the neural network proposed by word2vec. This procedure, as
noted in~\cref{sec:word2vec}, does consist of a linear map $\varphi$ from
$I_{L}$ to $H_{L}$, which is also the final embedding we are looking for in our
work. Note that it is not possible to represent a closure operator using a
linear function, since the closure of the empty set is not necessary an empty
set. The same fact is true for affine mappings, as showed in the following.

\begin{proposition}
  \label{prop}
  Let $(G,M,I)$ be a formal context. The set of all affine linear mappings,
  which represent the closure operator on the attribute set in binary encoding,
  can be empty.
\end{proposition}

\begin{proof}
    \begin{figure}[t]
    \centering
    \begin{cxt}%
      \att{1}%
      \att{2}%
      \att{3}%
      \obj{.xx}{a}%
      \obj{x.x}{b}%
      \obj{.x.}{c}%
    \end{cxt}
    \caption{A formal context counterexample for~\cref{prop}.}
    \label{fig:counterexample}
  \end{figure}
  Consider the formal context from \Cref{fig:counterexample}. For the sake of
  simplicity we speak about attribute and object sets and their respective
  binary encodings interchangeably. Assume that there is an affine mapping,
  which maps each attribute set to its closure. Then there is a linear mapping
  $l$, such that for each attribute set $v \in \{0,1\}^{|M|}$ the vector $v'=[1
  \ v]$ is mapped to the closure of $v$. Here $[1 \ v]$ denotes the vector which
  results from the concatenation of a single bit (valued 1) with $v$. Using this
  one can infer that from
  \begin{align*}
  &\{\}''= \{a,b,c\}' = \{\}\\
  &\{3\}''=\{a,b\}'=\{3\}\\
  &\{1,2\}''=\{\}'=\{1,2,3\}\\
  &\{1,2,3\}''=\{\}'=\{1,2,3\}
  \end{align*}
  follows that
  \begin{align*}
    &l(1,0,0,0)=(0,0,0)\\
    &l(1,0,0,1)=(0,0,1)\\
    &l(1,1,1,0)=(1,1,1)\\
    &l(1,1,1,1)=(1,1,1).
  \end{align*}
  However, as $l$ is a linear mapping, it is required that the following holds. 
  \begin{align*}
    l(1,1,1,1)&=l(1,1,1,0)+l(1,0,0,1)-l(1,0,0,0)\\
              &=(1,1,1)+(0,0,1)-(0,0,0)\\
              &=(1,1,2),
  \end{align*}
  Hence, we obtain a contradiction. \null\hfill$\square$
\end{proof}

\begin{corollary}
 Let $(G,M,I)$ be a formal context. The set of all affine linear mappings,
  which represent the derivation operator on the attribute set in binary encoding,
  can be empty.
\end{corollary}
\begin{proof}
  Assume there is such an affine linear map $a$. By duality we know that there
  must be an affine linear map $a^{d}$ on the object set. A suitable composition
  of those mapping, i.e., using augmentation, contradicts with
  \Cref{prop}. \null\hfill$\square$
\end{proof}

\subsubsection{Linear Representable Part of Closure Operators}
\label{sec:linearization}

We know from the last section that it is neither possible to represent the
closure operator nor the derivation operator using an affine linear
function. Still, it might be possible to obtain a meaningful approximation of an
embedding using a linear map. In order to obtain some empirical evidence if
studying this approach is fruitful we conduct a short experiment.  Consider the
neural network architecture as depicted in~\Cref{fig:2vec} (left). Furthermore,
let the input layer $I_{L}$ of size $|M|+1$ be connected to a hidden layer
$H_{L}$ of low dimension, i.e., two or three, by a linear function
$\varphi$. The layer $H_{L}$ is connected to the output layer $O_{L}$ that is of
dimension $M$ using a function $\psi$ that consists of a linear function
together with a \emph{sigmoid} activation function. The first bit of $I_{L}$ is
always set to 1 and therefore a so-called bias unit. For our experiment we now
train the neural network by showing it randomly sampled attribute sets as inputs
and their attribute closures as output, both in their binary encoding. We employ
for this mean squared error as the loss function and a learning rate of
0.001. Even though the neural network starts to memorize the samples it has seen
after around 20 epochs, it does not generalize to attribute sets not previously
seen in training. Furthermore, the resulting embedding into $\mathbb{R}^d$ does
empirically not expose a meaningful structure. Additionally, this observation
does not alter by changing the function $\psi$ to a linear function. Also,
experiments in which we investigated learning only the derivative operator were
not fruitful. This is the expected behavior from our considerations in the last
section. We do not claim that there are no better performing, approaches for
this task. However, for us this result motivate a progression to different task.

\subsubsection{Non-linear Embedding through Closure2Vec}
\label{sec:nonlinear}

As linear embeddings do not seem to work out for our learning task, we employ a
different approach. Let the \emph{closure Hamming distance (chd)} for two
attribute sets $A\subseteq M$ and $B\subseteq M$ be the distance function
$d(A,B)\coloneqq d_{H}(A^{b},B^{b})$, where $\cdot^{b}$ denotes the binary
representation and $d_{H}$ is the Hamming distance. Note that the closure
Hamming distance is not a metric, as the distance between two attribute sets
sharing the same closure is 0, even though they are not the same. Based on the
idea that two attribute sets are similar if they have a small chd, we want to
embed the attribute sets into a low-dimensional real-valued space, i.e., two or
three dimension. The goal here is that the embedding is approximately an
isometric map.

\begin{figure}[t]
  \centering
    \centering
    \includegraphics{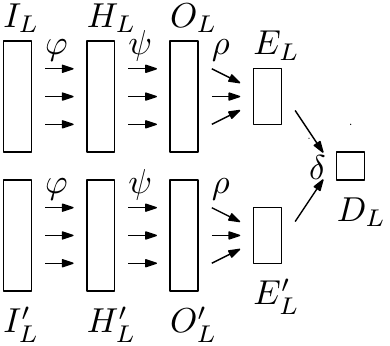}
    \caption{The siamese neural network used to compute \emph{closure2vec}. The
      functions $\varphi, \psi, \rho$ are shared functions between the layers in
      this model. }
    \label{fig:siamese}
\end{figure}

We train a neural network architecture that we call \emph{closure2vec} to learn
the just introduced chd. For this, consider the network depicted in
\Cref{fig:siamese}. It consists of two input layers $I_{L}$ and $I_{L}'$, each
of size $|M|$. Then the function $\varphi$ consisting of a linear function and a
\emph{relu-activation} function (see~\cite{lecun2015learning}) is used to feed
the data into the hidden layers $H_{L}$ and $H_{L}'$ respectively, both of size
$|G|$. After this the function $\psi$, consisting of a linear function and a
relu-activation function is applied to the two ``streams'' in the network. The
result then is input for two output layers $O_{L}$ and $O_{L}'$, both of size
$|M|$. This, however is not the final step of this network model. Finally, the
layers $E_{L}$ and $E_{L}'$, which consist of either two or three dimensions,
are fed from $O_{L}$ and $O_{L}’$, respectively, via $\rho$. This function is
again build via composing a linear function and another relu-activation
function. The output layer $D_{L}$ consists has size one. Using a fixed
function $\delta$ (in our case either the Euclidean distance or cosine distance)
we compute a distance between $O_{L}$ and $O_{L}'$. By sharing the functions
$\varphi$, $\psi$, and $\rho$ between the different layers we ensure that a
commutation of the input sets does not lead to a different prediction of the
neural network.

The network then is trained by showing it two attribute sets in binary encoding
as well as their closure Hamming distance at the output layer. The required loss
function for this setup is then the mean squared error. The learning rate of our
network is set to 0.001. The training set for our approach is sampled as
follows: For some $t\in\mathbb{N}$ take all attribute combinations that contain
at most $t$ elements and put them in some set $\mathcal{X}=\{X_1,X_2,\dotsc\}$,
hence, $X_{i}\subseteq M$. For each $X_i\in \mathcal{X}$ generate a random
attribute $m_{i} \in M$. Let the set $\mathcal{Y}=\{Y_1, Y_2,\dotsc\}$ with
\[  Y_i=
  \begin{dcases}
    X_i \backslash \{m_i\} & \text{if }m_i \in {X_i},\\
    X_i \cup \{m_i\}&\text{else,}
  \end{dcases}\] and finally
$Z=\{z_{i}\coloneqq d(X_{i},Y_{i})/|M|\mid X_{i}\in
\mathcal{X},Y_{i}\in\mathcal{Y})\}$ as the set of pairwise closure Hamming
distances. The network is trained by showing it the binary encodings of $X_i$,
$Y_i$, and $z_i$. Note, the values of $Z$ are normalized. We will evaluate this
setup in~\cref{sec:experimentsfca} on different data sets and for relevant
notions of FCA.

\subsection{Object2Vec and Attribute2Vec}
\label{sec:modelobj}

The idea of adapting word2vec to non-text mining problems is a common approach
these days. Particular examples for that are node2vec~\cite{grover16} and
\emph{deepwalk}~\cite{perozzi14}. In the realm of networks it was shown that SG
based architectures for node embeddings can beat former approaches that use
classic graph measures. They significantly enhanced node classification and link
prediction~\cite{perozzi14, grover16} tasks. To do so, they interpret nodes as
words for their vocabulary, random walks through the graphs to generate
``sentences'', and then employ word2vec.

Analogously, we transfer word2vec to the realm of formal concept analysis. In
the following we present an approach to use the concepts of a given formal
context to generate embeddings of the object set or attribute set. Referring to
its origin, we name our novel methods \emph{object2vec} and
\emph{attribute2vec}, respectively. Since both methods will emerge to be dual to
each other for obvious reasons we only consider object2vec in the following. The
basic idea is to interpret two objects to be more close to each other, if they
are included in more concept extents together. Hence, the set of extents of a
formal context is used to generate a low dimensional embedding of the object set
$G$.

In the following we explain how to adapt the CBoW and the SG architecture to the
realm of formal concept analysis. We show how to generate (multi-) sets of
training examples from a given formal context. As an analogon for target word
context words we introduce target object and context object sets. From this we
can draw pairs as already done in CBoW and SG.

\subsubsection{SG and CBoW  in the Realm of Object2Vec}
Let $\mathbb{K}\coloneqq(G,M,I)$ be a (finite) formal context. The vocabulary is given
by $G=\{g_1,\dotsc g_n\}$. Furthermore, let be $\phi: G \to \dR^n, g_i \mapsto
e^i$ the one-hot encoding of our vocabulary (objects). We derive our
\emph{training examples} from the set
\begin{equation*}
  \TC(\mathbb{K})\coloneqq\{ (a, A\setminus \{a\})~|~a \in A, |G|>|A|>1, \exists B
  \subset M: (A,B) \in \mathfrak{B}(\mathbb{K})\}, 
\end{equation*}
where every element is a pair of a \emph{target object} and some extent in which
$a$ is element of. More specifically, we remove $a$ from this extent. We
interpret then $A\setminus\{a\}$ as the \emph{object context set}. The word
``context'' here refers to the word2vec approach and is not be confused with
``formal context''.  Note that we do not generate any training examples from the
concept $(G,G')$ since the extent $G$ does not provide any information about
the formal context.

\paragraph{The Skip-gram Architecture for Object2Vec}
In the SG model, the input and output training pairs generated from a target
object and an object context set, i.e., the elements $(t,C) \in \TC(\mathbb{K})$, are 
given by:

\begin{equation}
  \label{eq:SG}
  \TP_{\SG}(t,C)\coloneqq\{(\phi(t),\phi(c))~|~c \in C\}.
\end{equation}
Using this it is possible for some pairs $(t_1,C_1),(t_2,C_2) \in
\TC(\mathbb{K})$ where we have $(t_1,C_1) \neq (t_2,C_2)$ that
$\TP_{\SG}(t_1,C_1) \cap \TP_{\SG}(t_2,C_2) \neq \emptyset$. Hence, samples can
be generated multiple times in this setup. We account for this in our modeling,
as the reader will see in the presentation of the algorithm. To give an
impression of our modeling we furnish the following example. 

\begin{example}
  Consider the classical formal context from~\cite{fca-book} called ``Living
  beings and Water'', which we depicted in~\cref{fig:water}. We map the objects
  with the one-hot encoding: $\phi: G\mapsto\dR^8$, with
  $\phi(a) =e^1, \phi(b)=e^2,\dotsc,\phi(h)=e^8$. Using this we easily find a
  training sample from $\TC_{SG}$ which is generated two times. Consider the
  concepts $(\{a,f,g\}, \{4,5,6\})$ and $(\{f,g,h \},\{5,6,7\})$. The first
  concept generates the pairs of target object and object context sets
  $(a, \{f,g\})$ as well as $(f, \{a, g\})$ and $(g, \{a, f\})$. The second
  formal concept generates the pairs $(f, \{g,h\}), (g, \{f, h\})$ and
  $(h, \{f, g\})$. If we train in the SG architecture we derive from the pair
  $(f, \{a, g\})$ the training examples $(e^6,e^1)$ and $(e^6,e^7)$. Also, from
  the pair $(f, \{g,h\})$ we derive the examples $(e^6,e^7),(e^7,e^8)$. Hence,
  the training example $(e^6,e^7)$ is shown to be drawn at least twice per
  epoch.
\end{example}

\begin{figure}[h]
  \centering
  \begin{cxt}
    \cxtName{}
    \att{1}
    \att{2}
    \att{3}
    \att{4}
    \att{5}
    \att{6}
    \att{7}
    \att{8}
    \att{9}
    \obj{...XXX..X}{a}
    \obj{xxx..x...}{b}
    \obj{xx.x.x.X.}{c}
    \obj{xxxx.x...}{d}
    \obj{X.X..X...}{e}
    \obj{...xxxx..}{f}
    \obj{..xxxxx..}{g}
    \obj{..x.xxx..}{h}
  \end{cxt}
  \caption{Formal context of the classical ``Living beings and Water'' example
    from~\cite{fca-book}.}
  \label{fig:water}
\end{figure}



\paragraph{The CBoW Architecture for Object2Vec}
Analogously to the cases of CBoW in word2vec we will use for object2vec a notion
of ``middle point'' for object context sets. More specifically, for a pair of
target object and object context set $(t,C) \in
\TC(\mathbb{K})$ the training example is derived as follows:
\begin{equation*}
  \TP_{\CBoW}(t,C) \coloneqq \left(\frac{1}{|C|}\sum_{c \in C} \phi(c),
  \phi(t)\right)
\end{equation*}
Hence, in the CBoW model the set of all training examples is given by
\begin{equation}
  \label{eq:CBoW}
  \TP_{\CBoW}(\mathbb{K})\coloneqq \{\TP_{\CBoW}(t,C)~|~(t,C) \in \TC(\mathbb{K}) \}.
\end{equation}

\begin{lemma}
\label{lemma-cbow}
The map $\TP_{\CBoW}: \TP (\dK) \to \dR^2, (t,C) \mapsto \left(\frac{1}{|C|}\sum_{c \in C} \phi(c),
  \phi(t)\right)$ is injective.
\end{lemma}
\begin{proof}
  Let $E_n$ be the set of standard basis vectors. We first show that
  \begin{equation*}
    f:2^{E_n} \to \dR^n, E \mapsto \frac{1}{|E|}\sum_{e \in E} e
  \end{equation*}
  is injective. Let $ E_1, E_2 \in 2^{E_n}$ with $E_1 \neq E_2$ where we assume
  $E_1 \not\subset E_2$ w.l.o.g.. Hence, we find $e^i \in E_1 \setminus E_2$.
  It then follows $f(E_1)_i > 0 =f(E_2)_i$, therefore $f(E_1) \neq f(E_2)$ and
  $f$ is injective. The function $\phi: G \to E_n$ is also injective, so the map
  $\Phi: 2^G \to 2^{E_n}, A \mapsto \Phi(A)\coloneqq\phi [A]$ is also injective. For all
  $(t,C) \in \TP(\dK)$ the equality
  \begin{equation*}
    \TP_{\CBoW}(t,C)=\left(f(\Phi(C)),\phi(t)\right)
  \end{equation*}
  holds. Hence, $\TP_{\CBoW}$ is injective.  \null\hfill$\square$
\end{proof}

It follows that the modeling of training samples as set (cf.~\cref{eq:CBoW}) is
approbeate since no training example is derived multiple times from $\TP(\dK)$.

\paragraph{Order of the training examples}
Let $\mathbb{K}\coloneqq(G,M,I)$ be a formal context. We want to embed the set
of objects. Since we model our training examples as sets with frequency (in the
case of SG) we need to discuss how to construct a traversable list of training
examples for our training procedures.  This is not necessary in word2vec where
the order is given naturally by the order of the given text. We propose to
generate the traversable list in the following manner:
\begin{enumerate}
\item For all extents in $A$ of $\dK$, construct a list $L_A$ that consists of
  all elements of $A$. The order in the list $L_A$ should be random.
\item Construct a list $L_{\ext}$ that consists of all lists $L_A$ in a random order.
\item For each $L_A$  use \cref{eq:CBoW} or \cref{eq:SG} to add the training
  examples to the list of all training examples.
\end{enumerate}
We present an algorithmic representation of this course of action
in~Algorithm~1.

\begin{algorithm}[t]
  \caption{The pseudocode of object2vec. The algorithm takes a formal context
    and an option determining Skip-gram or Continous Bag of Words. It returns a
    list of  pairs to train the neural network.}
  \DontPrintSemicolon \SetKwInOut{Input}{Input} \SetKwInOut{Output}{Output}
  \Input{a formal context $(G,M,I)$ and type $\in \{\SG, \CBoW\}$} \Output{A
    list $L$ of training examples.}
  \BlankLine $L \leftarrow [~]$ \\
  $L_{\ext} <-$ list-of-extents(G,M,I) in randomized order (excluding $G$). \\
  \ForAll{A in $L_{\ext}$}{
    $L_A \leftarrow$ list$(A)$, with randomized order.\\
    \ForAll{$o$ in $L_A$}{
      \If{type $=\SG$} {
        \ForAll{$\check{o}$ in $L_A$:}{
          \If{$o \neq \check{o}$}{
            add ($(\phi(o),\phi(\check{o})$)) to $L$}}}
      \If{type $=\CBoW$ and $| A | >1$}{
        add ($(\frac{1}{|L_A|-1}\sum_{\check{o} \in l_a, o \neq \check{o}} \phi(\check{o}),
        \phi(o))$) to $L$}
    }}
  \Return $L$
  \label{alg:algo}
\end{algorithm}


\section{Experiments}
\label{sec:experiments}
This section contains experimental evaluations for both our research directions.
We conduct our experiments on three different data sets. We depict the
statistical properties of these data sets in~\cref{fig:datasets}. A detailed
description of each follows.

\paragraph{wiki44k}
The first data set we use in this work is the wiki44k data set taken 
from~\cite{DBLP:conf/semweb/HoSGKW18} and then adapted by
\cite{DBLP:conf/icfca/Hanika0S19}. It consists of relational data extracted from
Wikidata in December 2014. Even though the it is constructed to be a
dense part of the Wikidata knowledge graph, it is relatively sparse for
a formal context.

\paragraph{Mushroom}
The Mushroom data set~\cite{Dua:2019,schlimmer1981mushroom} is a well
investigated and broadly used data set in machine learning and knowledge
representation. It consists of 8124 mushrooms. It has twenty two nominal features
that are scaled into 119 different binary attributes to form a formal
context. The Mushroom data set, compared to wiki44k, is more dense, and even
though it has a smaller number of objects, contains 10 times the concepts of
wiki44k.

\paragraph{ICFCA} To generate the ICFCA context, we use the DBLP dump from
2019-08-01 which can be found at \url{https://dblp.uni-trier.de/xml/}. We
exclude authors in the DBLP data that have the type ``disambiguation'' or
``group''. As attributes we use all publications of these authors. By all
publications we denote all publications present in DBLP, not restricted to ICFCA
proceedings. To exclude publication originating from editing etc (which do not
indicate any co-authorship) we discard all publications that are not of the
type ``article'', ``inproceedings'', ``book'' or ``incollection''. We also
discarded all publications that are marked with an addtitional ``publtype'' such
as ``withdrawn'', ``informal'' and  ``informal withdrawn''. Note that this also
excludes works that are solely published on preprint servers. This modeling
results in a formal context with 351 objects and 12614 attributes. However, as
later indicated in the experiments, for the neural network training we will use
only a part of that formal context. This part is derived by omitting all
publications after 2015 and then considering the largest component. The
specifications of the resulting formal context ICFCA* are depicted
in~\cref{fig:datasets}. The ICFCA data set is available in the
\texttt{conexp-clj}\footnote{see~\url{https://github.com/tomhanika/conexp-clj}}
software~\cite{DBLP:conf/icfca/HanikaH19} for FCA hosted on GitHub. By the
nature of being based on a publication network, it is very sparse and contains
only 878 concepts.

\begin{table}[t]
  \centering
  \caption{Comparison of the different data sets used in this work. For ICFCA we
    do only indicate the specification of the context as used for the training
    in the link prediction model. To compute the canonical base of the ICFCA
    data set as is not feasible for the equipment at our research group.}
  \begin{tabular}{lrrr}
    \toprule
    &Wiki44k&Mushroom&ICFCA*\\
    \midrule
    Number of Objects&45021&8124&263\\
    Number of Attributes&101&119&8442\\
    Density&0.04&0.19&0.005\\
    Number of Concepts&21923&238710&680\\
    Mean attributes per concept&7.01&16.69&33.28\\
    Mean objects per concept &109.47&91.89&2.51\\
    Size of the Canonial Base& 7040 & 2323& ? \\
    \bottomrule
    
  \end{tabular}
  \label{fig:datasets}
\end{table}

\subsection{Object2Vec and Attribute2Vec}
\label{sec:obj2vec}

We evaluate our new approaches object2vec and attribute2vec with two distinct
experiments. First, we will study embeddings in the realm of link
prediction. For this, we investigate a self created publication network as
described in the last section, called ICFCA. In this formal context, consisting of
authors as objects and publications as attributes, the incidence relation is
then given by \emph{g is author of m}. Link prediction tasks can be split into
two categories: decide in a network which links are missing or predict from a
given temporal network snapshot which new links will occur in the future. In
general this experiment evaluates the ability of object2vec to enhance link
prediction.

In our second experiment we present a task that is of more general interest to
formal concept analysis concerned research. We investigate a correspondence
between the canonical base of implication $\mathcal{L}$ for a given formal
context $(G,M,I)$ and our embedding methods. In particular, we cluster the set
of attributes $M$ based on attribute2vec using a partitioning procedure and obtain
a clustering $\mathcal{C}$. We then count the number of implications $A\to B$
from $\mathcal{L}$ that are in subset relation with a cluster, i.e., $A\cup
B\subseteq C$ for some cluster $C\in\mathcal{C}$. With that we evaluate to which
extend attribute2vec embeddings are able to reflect parts of the implicational
structure from a formal context.

\subsubsection{Link Prediction using Object2Vec}

Network embedding techniques like the prominent node2vec approach have proven
their capability to predict links in huge networks~\cite{grover16}. Even though
these methods employ low dimensional embeddings for their computations, the
actual employed dimension is still incomprehensible high for human
understanding, i.e., more than 100. The realm of formal concept analysis is
especially interested in interpretable and explainable methods. Hence, we focus on
embeddings into $\dR^2$ and $\dR^3$.

Using the afore described ICFCA data set our goal for the link prediction task
is as follows. For learning an embedding we restrict ICFCA to the largest
connected component we discover after we omitted all publications later than
2015. The prediction task then is to find future co-authorships. More
specifically, we predict these co-authorships in the time interval from
2016-01-01 until 2019-08-01. We compare the introduced object2vec using both
architectures, i.e., CBoW and SG, and compare the results with link prediction
computed via node2vec. We may note that the node2vec embeddings are conducted in
two and three dimensions as well. Our precise experimental pipeline looks as follows.
\begin{description}
\item[Compute an embedding via object2vec] We first compute an embedding of the
  formal context with the object2vec approach. For the training of the neural
  network we use a starting learning rate of $1.0$ with linear decrease. We
  train for 200 epochs. We repeat the embedding process for 30 times.
\item[Compute the node2vec embedding] The node2vec embedding uses the parameters
  as used in~\cite{grover16}, i.e., $10$ walks per node with length $80$ and a
  window size of $10$. We use the standard learning rate of $0.025$ as well as
  $1.0$ for comparability reasons with respect to object2vec.  However, we do
  not report the results for 1.0 since they were worse for the two and three
  dimensional case. We also repeat the embedding process for 30 times. The
  procedure parameters $p$ and $q$ of node2vec are chosen by grid search in
  $\{0.25,0.5,1,2,4\}$.
\item[Edge vector generation] Since link prediction is concerned with edges we
  need to focus on the edge set. To generate the edge vectors from the node
  embedding we use the componentwise product of two node
  vectors. In~\cite{grover16} it was shown that this practice is
  favorable. Furthermore, in the special case of bibliometric link prediction
  this is the common approach, cf.~\cite{ganguly17}.
\item[Training examples] Our learning procedure demands for training
  examples. The positive training examples are the edges in the co-author graph
  until (including) 2015. For each positive example, we select one negative
  example, i.e., two randomly picked nodes without an edge connecting them. This
  approach leads to 1278 training examples.
\item[Test examples] The positive examples are the author pairs which have an
  edge after 2016-01-01, but not before. For each such pair, we choose one
  negative example of two authors that neither co-authored before or after 2016.
  Using this we obtain 84 test examples, half of them positive. 
\item[Classification] For the binary classification problem, we employ logistic
  regression. In particular, we use the implementation provided by \emph{Scikit
    Learn}~\cite{scikit-learn}. Logistic regression is often used in
  classification emerging from embeddings in the realm of word2vec, e.g.,
  in~\cite{grover16, perozzi14}. To determine the C parameter of the classifier, we do a
  grid search over $\{10^{-3},10^{-2},\dotsc, 10^2\}$.
\end{description}


\begin{table}[b]
  \centering
    \caption{The result of our classification experiment. We compare node2vec to
    object2vec with the Skip-gram architecture (O2V-SG) and with the object2vec
    Continuous Bag of Words architecture (O2V-CBoW). We display the mean value
    over the 30 rounds of the experiments and also present the sample standard deviation.}
  \begin{tabular}{l|c|ll|ll|ll|}
    \toprule
    Embedding & Dim & \multicolumn{2}{l}{Recall} & \multicolumn{2}{l}{Precision} &
                                                                             \multicolumn{2}{l}{F1-Score}
    \\
         &     & mean & stdev & mean & stdev &mean & stdev \\
    \midrule
    node2vec &2 & 0.56& 0.14 & 0.60 & 0.07  & 0.57 & 0.09\\
    O2V-SG   &2& 0.66 & 0.08 & \textbf{0.65}& 0.03 & 0.65 & 0.05\\
    O2V-CBoW &2& \textbf{0.68} & 0.09 & 0.64 & 0.04 & \textbf{0.66} 	 & 0.06  \\
    \midrule
    node2vec &3 & 0.60& 0.15  & 0.56 & 0.08 & 0.58 	& 0.10\\
    O2V-SG   &3& 0.70 & 0.08 &  	0.62 &  0.04&  	0.66  & 0.06\\
    O2V-CBoW &3& \textbf{0.73} & 0.07 &\textbf{0.65} & 0.06 &\textbf{0.69}  & 0.06 \\
    \bottomrule
  \end{tabular}
 
    \label{tab:classification}
\end{table}

The results of our experiments are depicted in~\cref{tab:classification}. We
observe that in all three indicators, i.e., recall, precision and F1-Score,
node2vec is dominated by the object2vec approach. Furthermore, we see that the
CBoW architecture performs better compared to SG in almost all cases. However,
this benefit is small and well covered in the standard deviation. Finally, we
find that embeddings in three dimensions perform in generally better than in
two.

\subsubsection{Clustering Attributes with Attribute2Vec}

In FCA, implications on the set of attributes of a formal context are of major
interest. While computing the canonial base, i.e., the minimal base of the
implicational theory of a formal context, is often infeasible, one could be
interested in implications of smaller attribute subsets. This leads naturally to
the question of how to identify attribute subsets that cover a large part of the
canonial base, as explained in the beginning of~\cref{sec:obj2vec}. In detail,
the resulting task is at follows: Let $(G,M,I)$ be a formal context and
$\mathcal{L}$ the canonical base $(G,M,I)$. Using a simple clustering procedure
(in our case $k$-means), find a for a given $k \in \dN_0$ partitioning of $M$ in
$k$ clusters such that the ratio of implications completely contained in one
cluster (cf.~\cref{sec:obj2vec}) is as high as possible. We additionally
constraint this task by limiting clusterings in which the largest cluster
is significantly smaller than $|M|$.

We investigate to which extend our proposed approach attribute2vec maps
attributes closely that are meaningful for the afore mentioned task.  We conduct
this research by computing an embedding via attribute2vec and run the $k$-means
clustering algorithm on top of it. We evaluate our approach on the introduced
wiki44k data set. Again, we refer the reader to the collected statistics of this
data set in~\cref{fig:datasets}.  The experimental pipeline looks as follows.

\begin{description}
\item[Applying attribute2vec on wiki44k] We start by computing two and three
  dimensional vector embeddings of the wiki44k attributes using
  attribute2vec. We employ both architectures, SG and CBoW. Here we use again
  the learning rate of $1.0$. In contrast to the embeddings of the ICFCA data
  set we find that $5$ training epochs are sufficient for stabilization of the
  embeddings.
\item[$K$-means clustering] We use the computed embedding to cluster our
  attributes with the $k$-means algorithm. As implementation we rely on the
  Scikit Learn software package. For the initial clustering, we use the so
  called ``k-means++'' technique by~\cite{DBLP:conf/soda/ArthurV07}. The method
  from Scikit Learn runs internally for ten times with different seeds and
  returns the best result encountered. We choose $k$
  from $\{2, 5, 10\}$. We denote the resulting clustering with $\mathcal{C}$.
\item[Computation of the intra-cluster implications] An implication drawn from
  the canonical base, i.e., $A\to B\in\mathcal{L}$, is called
  \emph{intra-cluster} if there is some $C\in\mathcal{C}$ such that $A\cup
  B\subseteq C$. The canonical base of wiki44k has the size 7040. For a
  clustering $\mathcal{C}$ we compute the ratio of intra-cluster implications.
\item[Repetition] We repeat the steps above for 20 times. Hence, we report the
  mean as well as the standard deviation of all the results.
\item[Baseline clusterings] To evaluate the ratios computed in the last step we
  use the following baseline approaches. As a first baseline we make use of a
  random procedure. This results in a random clustering of the attribute set.
  Using an arbitrary random clustering with respect to cluster sizes is
  unreasonable for comparison. Hence, for each $k$-means clustering obtained
  above we generate $50$ random clusterings of the same size and the same
  cluster size distribution. For those we also compute the intra-cluster
  implication ratio. 
\item[Naive $k$-means clustering] As second baseline we envision a more
  sophisticated procedure. We call this the ``naive'' clustering approach.  In
  this setting we encode an attribute $m$ through a binary vector representation
  using the objects from $\{m\}’$ as described in~\cref{sec:retrieving}.  We
  then run twenty rounds of $k$-means and compare the results with the
  attribute2vec approach. For comparison we use here again $k \in \{2, 5, 10\}$.
\end{description}

We display our observations from this experiment in~\cref{tab:cluster}. In there
we omit results for the CBoW architecture since they do not exceed the results
obtained in the random baseline approach. As our main result we find that the
Skip-gram architecture in three dimensions achieves the best intra-cluster
ratio. More specifically, SG outperforms for all cluster sizes and all dimension
the baseline approach and the naive clustering approach by a large margin. This
margin, however, is smaller in the two dimensional case compared to the three
dimensional try. For smallest investigated clustering size, i.e., two, naive
clustering performs worse than the random baseline. We can report for our
experiments the following average maximum cluster sizes. For dimension three we
have 54.5 attributes for $k=2$, 33.2 attributes for $k=5$, and 14.9 attributes
for $k=10$. We also spot for clustering sizes five and ten that the naive
clustering does operate better than the random baseline. Finally, we note that
the stability of the SG result in the $k=2$ case sticks out compared to the
results of the competition.

\begin{table}[t]
    \caption{Results of the clustering task. For the dimensions 2 and 3 we show
    the mean and sample standard deviation values for the Skip-gram
    architecture and the random clusters with same cluster size distributions as
  the corresponding Skip-gram clusterings. We also compared our method with the
  ``naive'' clustering approach.}
  \centering
  \begin{tabular}{lc|ll|ll|ll}
    \toprule
    \# clusters&  & \multicolumn{2}{c}{2}  & \multicolumn{2}{c}{5} & \multicolumn{2}{c}{10}
    \\
    \midrule
    type &dim & mean &stdev &mean &stdev &mean &stdev \\
    \midrule
    Skip-gram & 2& 0.1608 & 0.0031 & 0.0703 & 0.0122 & 0.0069& 0.0004 \\
    Random & 2 & 0.0534& 0.0412& 0.0084& 0.0088& 0.0010 & 0.0007 \\
    \midrule
    Skip-gram &3 & \textbf{0.3217} & 0.0005 & \textbf{0.1028} 	& 0.0218 & \textbf{0.0080} & 0.0001    \\
    Random & 3 & 0.0219 & 0.0107 & 0.0036 & 0.0027 & 0.0007 & 0.0004\\
    \midrule
    Naive Clustering & - & 0.0158 & 0.0000 & 0.0055 & 0.0042 & 0.0035
    & 0.0002 \\
    \bottomrule
  \end{tabular}
  \label{tab:cluster}
\end{table}

\subsubsection{Discussion}
In both our experiments we find that all embedding procedures perform better in
dimension three than in dimension two. This is not surprising since a higher
embedding dimension posses a higher degree of freedom to represent
structure. Furthermore, in both experiments we can show that the object2vec and
attribute2vec approaches do succeed and outperform the competition.

The first experiment reveals some particularities. We find that our embedding
approach has a big advantage over the also considered node2vec procedure. For
the later method one has to perform additional parameter tuning for the $p$ and
$q$ parameters~\cite{grover16} to obtain the presented results. Not doing so
leads worse to performance. The object2vec embedding procedure, as defined
in~\cref{sec:modelobj}, needs no parameters for the training example
generation. In addition to that is the set of computed training examples
deterministic. However, these positive properties come with a high computational
cost, i.e., the necessity of computing the set of formal concepts. For the sake
of completeness we also report on high dimensional embeddings. When applying
node2vec with embedding dimension one hundred we find the results outperforming
the so far reported. Still, since such embeddings conflict with the goal in this
work, i.e., the human interpretability and explainability of embeddings, we
discard them. 

The second experiment also unraveled different properties of our novel embedding
technique. We witness that the number of learning epochs is much smaller
compared to the first experiment. We suspect that this can be attributed to the
higher average number of attributes per intent in wiki44k compared to the
average number of objects per extent in ICFCA. Furthermore, we think there is an
influence by the fact that the absolute number of intents in wiki44k is greater
than the absolute number of extents in ICFCA. As application for our
attribute2vec approach we envision the computation of parts of the canonical
base of a formal context. Taking the average maximum cluster sizes into account
we claim that this application is reasonable. We do not consider the then
necessary computation of all formal concepts as an disadvantage. The computation
of the canonical base is in general far more complex than computing the set of
all concepts and our embedding. We are surprised that the naive clustering
baseline performs not significantly better than the random baseline. We presume
that the employed distance function from $k$-means applied to the binary
representation vectors is not useful to reflect implicational knowledge in the
embedding space. Hence, we admit that more powerful base line comparisons may be
considered here. However, so far we are not aware of less computational
demanding ones with respect to object2vec and attribute2vec respectively.

\subsection{FCA Features Through Closure2Vec}
\label{sec:experimentsfca}
To evaluate the embeddings produced by closure2vec we introduce two FCA related
problems: computing the covering relation of a concept lattice and computing the
canonical base for a given formal context. The intention here is to rediscover
structural features from FCA in low dimensional embeddings. We choose for the
dimension two and three in order to respect our overall goal for human
interpretability and explainability. We test two different functions $\delta$
for the distance between the output layers $O_{L},O_{L}’$,
cf.~\cref{sec:nonlinear}. More specific, we employ the Euclidean distance and
the cosine distance. We conduct our experiments on two larger than average
sized formal contexts.  Precisely, we test the
Wiki44k\cite{DBLP:conf/icfca/Hanika0S19,DBLP:conf/semweb/HoSGKW18} and the well
investigated Mushroom data set\cite{Dua:2019,schlimmer1981mushroom}. A
comparison of the statistical properties of data sets we use is depicted
in~\cref{fig:datasets}.

\subsubsection{Distance of Covering Relation}
\label{sec:coverrelation}

For this experiment we compute first the set of all concepts $\mathfrak{B}$ of
a given formal context $\mathbb{K}$. Using the concept order relation $<$ as
introduced in~\cref{sec:FCA} a \emph{covering relation} on $\mathfrak{B}$ is given
by: ${\prec}\subset(\mathfrak{B}\times\mathfrak{B})$ with $A \prec B$, if and
only if $A < B$ and there is no $C \in \mathfrak{B}$ such that $A<C <B$. The
covering relation is an important tool in \emph{ordinal data science}. Elements of
the covering relation are essential for investigating and understanding order
relations and order diagrams. However, in the case of large formal contexts
computing the covering relation of the concept lattice can get computationally
expensive, as this problem is linked to the transitive reduction of a
graph~\cite{doi:10.1137/0201008}.

The experimental setup now is as follows. First we train the neural network
architecture as introduced in~\cref{sec:nonlinear}. Hence, an input element is a
binary encoded attribute $X_{i}$, another binary encoded attribute set $Y_{i}$
of size $|X_{i}|\pm 1$, and the closure Hamming distance of them. In our
experiment we fix $|X_{i}|$ to be four or less. Furthermore, we train the
network over five epochs using the learning rate 0.001, with batch size 32, and
mean-squared-error as loss function.

To evaluate the structural quality of the obtained embedding we computed the
covering relation of the concept lattices using 1000 threads on highly parallelized
many-core systems, which took about one day. In the following we compare the
distances between pairs of concepts in covering relation against concepts that are
not in covering relation. The results of these experiments is depicted
in~\cref{fig:cover_relation}. For all embeddings the expected distance for two
concepts in covering relation a significantly smaller than the expected difference
of two concepts not in covering relation. This is true for both data sets. However,
the observed effect is more notable for the wiki44k data set. The Euclidean
distance outperforms the cosine distance in all experiments using two and three
dimensions.

\begin{table}[t]
  \centering
  \caption{Distance between concept pairs, that are in covering relation (CR)
    and that are not in the covering relation (Non-CR).}
  \textbf{Wiki44k:}\\
  \begin{tabular}{llrr}
    \toprule
    \textbf{Dim 2}&& Mean: & Std.: \\
    \midrule
    Euk:&CR:& 0.17 & 0.14 \\
    &Non-CR:& 0.71 & 0.59\\
    \midrule
    Cos:&CR:& 0.63 & 0.33 \\
    &Non-CR:& 0.99 & 0.71\\
    \bottomrule
  \end{tabular}
  \begin{tabular}{llrr}
    \toprule
    \textbf{Dim 3}&& Mean: & Std.: \\
    \midrule
    Euk:&CR:& 0.16 & 0.15 \\
    &Non-CR:& 1.54 & 1.41\\
    \midrule
    Cos:&CR:& 0.15 & 0.27 \\
    &Non-CR:& 0.36 & 0.43\\
    \bottomrule
  \end{tabular}

  \medskip
  
  \textbf{Mushrooms:}\\
  \begin{tabular}{llrr}
    \toprule
    \textbf{Dim 2}&& Mean: & Std.: \\
    \midrule
    Euk:&CR:& 0.14 & 0.41 \\
    &Non-CR:& 0.51 & 0.38\\
    \midrule
    Cos:&CR:& 0.49 & 0.43 \\
    &Non-CR:& 0.96 & 0.74\\
    \bottomrule
  \end{tabular}
  \begin{tabular}{llrr}
    \toprule
    \textbf{Dim 3}&& Mean: & Std.: \\
    \midrule
    Euk:&CR:& 0.13 & 0.29 \\
    &Non-CR:& 0.49 & 0.41\\
    \midrule
    Cos:&CR:& 0.05 & 0.12 \\
    &Non-CR:& 0.18 & 0.22\\
    \bottomrule
  \end{tabular}
  \label{fig:cover_relation}
\end{table}

\begin{table}[t]
  \centering
  \caption{Distances of implication premises and conclusions for singleton
    implications (S-Impl) and implications (Impl) in the computed embeddings.}
    \textbf{Wiki44k:}\\
  \begin{tabular}{llrr}
    \toprule
    \textbf{Dim 2}&& Mean: & Std.: \\
    \midrule
    Euk:&S-Imp:& 0.94 & 0.28 \\
    &Non-S-Imp:& 0.73 & 0.40\\
    &Imp:& 0.44& 0.33 \\
    &Non-Imp:& 0.51 & 0.48\\
    \midrule
    Cos:&S-Imp:& 1.00 & 0.69 \\
    &Non-S-Imp:& 1.00 & 0.67\\
    &Imp:& 1.01& 0.70 \\
    &Non-Imp:& 1.01 & 0.70\\
    \bottomrule
  \end{tabular}
  \begin{tabular}{llrr}
    \toprule
    \textbf{Dim 3}&& Mean: & Std.: \\
    \midrule
    Euk:&S-Imp:& 0.93 & 0.29 \\
    &Non-S-Imp:& 0.83 & 0.47\\
    &Imp:& 0.60& 0.45 \\
    &Non-Imp:& 0.65 & 0.55\\
    \midrule
    Cos:&S-Imp:& 0.69 & 0.53 \\
    &Non-S-Imp:& 0.90 & 0.43\\
    &Imp:& 0.93& 0.56 \\
    &Non-Imp:& 0.96 & 0.58\\
    \bottomrule
  \end{tabular}

  \bigskip
  
  \textbf{Mushrooms:}\\
  \begin{tabular}{llrr}
    \toprule
    \textbf{Dim 2}&& Mean: & Std.: \\
    \midrule
    Euk:&S-Imp:& 0.95 & 0.41 \\
    &Non-S-Imp:& 0.45 & 0.32\\
    &Imp:& 0.70 & 1.01 \\
    &Non-Imp:& 1.02 & 0.98\\
    \midrule
    Cos:&S-Imp:& 1.00 & 0.65 \\
    &Non-S-Imp:& 1.00 & 0.65\\
    &Imp:& 1.00& 0.69 \\
    &Non-Imp:& 0.99 & 0.69\\
    \bottomrule
  \end{tabular}
  \begin{tabular}{llrr}
    \toprule
    \textbf{Dim 3}&& Mean: & Std.: \\
    \midrule
    Euk:&S-Imp:& 0.86 & 0.34 \\
    &Non-S-Imp:& 0.42 & 0.27\\
    &Imp:& 0.57& 0.66 \\
    &Non-Imp:& 0.88 & 0.68\\
    \midrule
    Cos:&S-Imp:& 1.05 & 0.36 \\
    &Non-S-Imp:& 1.02 & 0.35\\
    &Imp:& 0.88& 0.53 \\
    &Non-Imp:& 0.89 & 0.55\\
    \bottomrule
  \end{tabular}
  \label{fig:impl}
\end{table}

\subsubsection{Distance of Canonical Bases}

In this experiment we look at the canonical base of implications for a formal
context and try to rediscover this canonical base in the computed embedding. The
experiment consists of two different parts. The first part has the following
setup. Take an implication of the canonical base, i.e., take $(P,C)$ where
$P\subseteq M$ is the premise and $C\subseteq M$ is the conclusion. For such an
implication, construct all single conclusion implications $(P,c)$ where $c\in
C$. Then, compute the distance (with the same distance functions as used for the
embedding process) of $P$ and the embedding for all $c$. Additionally, also
embed all $m\in M\backslash C$. Essentially, by doing so, we embedded all $m\in
M$ using our embedding function. We do this for all elements of the canonical
base.

Now compute for all implications from the canonical base the following
distances. First, the distances between a premise $P$ and all its singleton
conclusions $c$ of $(P,C)$. Secondly, the distances between $P$ and $m\in
M\setminus C$. Equipped with all these distances we try to detect a structural
difference in favor of the embedded implications in contrast to other
combinations of attribute sets, i.e., pairs of premises $P$ and $m\in M\setminus
C$. When using the cosine distance function we observe minimal to no structural
difference. However, when using the Euclidean distance function we detect a
significant structural difference. In particular, for pairs of $(P,c)$ with
$c\in C$ the mean of the distances is significantly higher than the mean
distance of the distances for some premise $P$ with all singleton sets $m\in
M\setminus C$. The observation is even stronger in the case of the Mushroom data
set when compared to wiki44k. The results are depicted in~\cref{fig:impl} by the
rows S-Imp (for the combinations $(P,c)$) and Non-S-Imp (for combinations
$(P,m)$). 

For the second part we embed both attribute sets, i.e., the premise $P$ and the
conclusion $C$, for an implication from the canonical base. For every pair we
compute the distance of $P$ and $C$ and compare them the distance between two
randomly generated attribute sets $X,Y$ with $|X|=|P|$ and $|Y|=|C|$. The later
is set for reasons of comparability. As shown in~\cref{fig:impl} we detect again
structural differences for the considered implications and the randomly
generated sets using average distances as features. In fact, the distance
between the two randomly generated sets is on average larger than the average
distance between premises and conclusions from implications drawn from the
canonical base.



\subsubsection{Discussion}
In both experiments concerning the closure system embedding we are able to
rediscover and infer conceptual structures in the embeddings. In general we find
that it is favorable to use for the Euclidean distance case the squares of the
closure Hamming distances as output, i.e., for $z_{i}$. Overall we discovered a
significant bias in distances of embedded concepts that are in covering
relation. This signal is even stronger for the wiki44k data sets. We suspect
that this can be attributed to the lesser density of this data set compared to
Mushroom. The observations can be exploited naturally for mining covering
relations, or important parts of those, from embedded concept lattices.

For second experiment we can report that the neural network embedding of parts
of the closure system allows for rediscovering implicational structures. Since
we trained our neural network on attribute sets of size four and smaller, we
were interested in the number of closures our algorithm encounters. For both
data sets we can report that this number is approximately 10\% of all closures.

The described a differing behavior for embedding whole implications, i.e.,
premise and conclusions, and single-conclusion implications. In cases where an
attribute is element of the conclusion of some canonical base implication the
distance to the premise set is significant. At this point we are unable to
provide a rational for that. The same goes for the second part of the experiment
where we compare premise-conclusion pairs of canonical base implications with
randomly generated pairs of attributes having the same sizes. At this point we
are not aware how this observation can improve the computations of the canonical
base. This would need a more fundamental investigation of bases of implicational
theories with respect to closure system embeddings. For example, one has to
investigate if the observed effect is also true for other kinds of bases, e.g.,
direct basis~\cite{DBLP:journals/dam/AdarichevaNR13}.  As a final remark we
report that the Euclidean distance performed in all our experiments better than
the cosine distance for both problem settings.

\subsubsection{Empirical Structural Observations}

\begin{figure}[t]
  \centering
  \includegraphics[width=0.3\textwidth]{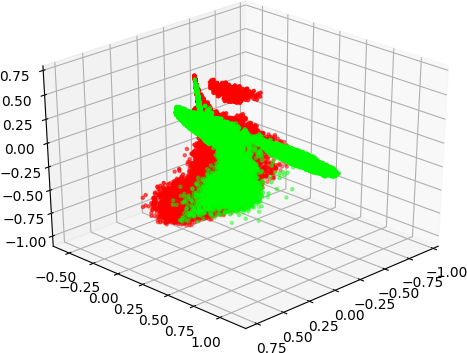} 
  \includegraphics[width=0.3\textwidth]{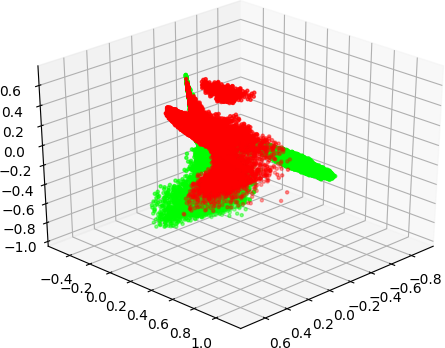}
  \includegraphics[width=0.3\textwidth]{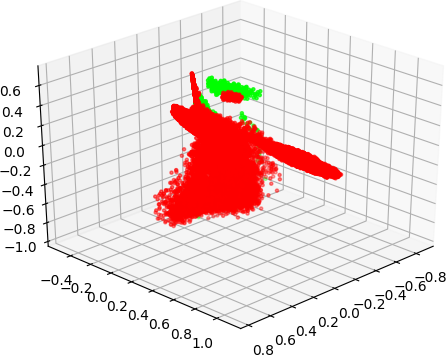}
  \caption{Embedding of all concepts of the mushroom dataset into three
    dimensions. The coloring is done as follows. Left: The edible mushrooms are
    green, the non-edible mushrooms are red. Middle: The mushrooms with a broad
    gill are green, the mushrooms with a narrow gill are red. Right: The
    mushrooms with a crowded gill spacing are green and the mushrooms with a
    distant gill spacing are red.}
  \label{fig:attr_clust}
\end{figure}

Additionally to the two experiments above we want to provide some insights we
discovered for conceptual structures in our embeddings.  We note that concepts
sharing attributes seem to result in meaningful clusters. To see this one can
consider \Cref{fig:attr_clust}. In there we see the same embedding of all formal
concepts of the Mushroom data set in three dimensions for three times. In each
case we colored different sets of concepts with red and green. In the first
(\Cref{fig:attr_clust}, left) we depict with red the not edible mushrooms and
with green the edible ones. Even though we employed a very low dimensional
embedding, we can still visually identify the two different classes. Hence, our
embedding approach preserved some structure. The same seems true for the other
depictions in which we colored broad gill versus narrow gill and crowded gill
spacing versus distant gill spacing. Therefore, we are confident that our
approach for low dimensional embeddings of closure systems using neural networks
is beneficial. Moreover, as this empirical study shows, is the low dimensional
representation still visitable by a human data analyst.

\section{Conclusion}
\label{sec:conc}
In this work we presented fca2vec, a first approach for modeling data and
conceptual structures from FCA to the realm of embedding techniques based on
word2vec. Taken together, the ideas in this paper outline an appealing
connection between formal concepts, closure systems, low dimensional embeddings,
and neural networks learning procedures. We are confident that future research
may draw on the demonstrated first steps, in particular object2vec,
attribute2vec and learning closure operator representations. In our
investigation we have found convincing theoretical as well as experimental
evidence that FCA based methods can profit from word2vec like embedding
procedures. We demonstrated that closure operator embeddings that result from
simple neural network learning algorithms may capture significant portions of
the conceptual structure. Furthermore, we were able to demonstrate that the
cover relation of the set of formal concepts may be partially extracted from a
low dimensional embedding.  Especially when employing conceptual structures in
large and complex data this notion is an important step forward. Moreover, we
were able to enhance the common embedding approach node2vec in low dimensional
cases, i.e., dimension two or three.

All these results were achieved while obeying the constraint for human
interpretable and/or explainable embeddings. Applying neural network learning
procedures on large and complex data does not necessarily constitute a
contradiction to explainability when combined with conceptual notions from
FCA. However, our work clearly has some limitations. The ideas for object2vec
and attribute2vec do require the computation of the concept lattice. In future
work we will investigate if this obligation can be weakened through statistical
methods. Despite this we believe that our work could be the standard framework
for word2vec like FCA approaches. As a next concrete application we are
currently in the process of investigating genealogy graphs in combination with
co-authorship networks. These multi-relational data sets are large and complex
and do require novel methods, like fca2vec, to draw knowledge from
them. Questions for the relation of particular nodes in such data sets may be
answered through conceptual embeddings.  In this context we do also take
Resource Description Framework (RDF) structures into account. Ideas for
embedding those is a state of the art approach to knowledge graph
structures. Hence, enhancing RDF embeddings using fca2vec as well as discovering
conceptual structures in RDF is a fruitful endeavor. Finally, on a more
technical note we are interested in characterizing sets of formal context data
allowing for particular representations of the closure operator, e.g., closure
operators representable by affine maps.

\subsection*{Acknowledgement}
This work is partially funded by the German Federal Ministry of Education and
Research (BMBF) in its program ``Quantitative Wissenschaftsforschung'' as part
of the REGIO project under grant 01PU17012, and in its program ``Forschung zu den
Karrierebedingungen und Karriereentwicklungen des Wissenschaftlichen Nachwuchses
(FoWiN)'' under grant 16FWN016.

\bibliographystyle{bib/splncs04} \bibliography{bib/paper}

\begin{thebibliography}{10}
\providecommand{\url}[1]{\texttt{#1}}
\providecommand{\urlprefix}{URL }
\providecommand{\doi}[1]{https://doi.org/#1}

\bibitem{DBLP:journals/dam/AdarichevaNR13}
Adaricheva, K.V., Nation, J.B., Rand, R.: Ordered direct implicational basis of
  a finite closure system. Discrete Applied Mathematics  \textbf{161}(6),
  707--723 (2013)

\bibitem{doi:10.1137/0201008}
Aho, A.V., Garey, M.R., Ullman, J.D.: The transitive reduction of a directed
  graph. SIAM Journal on Computing  \textbf{1}(2),  131--137 (1972).
  \doi{10.1137/0201008}

\bibitem{DBLP:conf/soda/ArthurV07}
Arthur, D., Vassilvitskii, S.: k-means++: the advantages of careful seeding.
  In: Bansal, N., Pruhs, K., Stein, C. (eds.) Proceedings of the Eighteenth
  Annual {ACM-SIAM} Symposium on Discrete Algorithms, {SODA} 2007, New Orleans,
  Louisiana, USA, January 7-9, 2007. pp. 1027--1035. {SIAM} (2007)

\bibitem{bishop06}
Bishop, C.M.: Pattern recognition and machine learning. Springer Science+
  Business Media (2006)

\bibitem{LSAapproach}
Codocedo, V., Taramasco, C., Astudillo, H.: Cheating to achieve formal concept
  analysis over a large formal context. In: Napoli, A., Vychodil, V. (eds.)
  Proceedings of The Eighth International Conference on Concept Lattices and
  Their Applications, Nancy, France, October 17-20, 2011. {CEUR} Workshop
  Proceedings, vol.~959, pp. 349--362. CEUR-WS.org (2011)

\bibitem{Dua:2019}
Dua, D., Graff, C.: {UCI} machine learning repository (2017),
  \url{http://archive.ics.uci.edu/ml}

\bibitem{ganguly17}
Ganguly, S., Pudi, V.: Paper2vec: Combining graph and text information for
  scientific paper representation. In: Jose, J.M., Hauff, C., Alt{\i}ngovde,
  I.S., Song, D., Albakour, D., Watt, S., Tait, J. (eds.) Advances in
  Information Retrieval. pp. 383--395. Springer International Publishing, Cham
  (2017)

\bibitem{fca-book}
Ganter, B., Wille, R.: Formal Concept Analysis: Mathematical Foundations.
  Springer-Verlag, Berlin (1999)

\bibitem{DBLP:journals/corr/GoldbergL14}
Goldberg, Y., Levy, O.: word2vec explained: deriving mikolov et al.'s
  negative-sampling word-embedding method. CoRR  \textbf{abs/1402.3722} (2014)

\bibitem{grover16}
Grover, A., Leskovec, J.: node2vec: Scalable feature learning for networks. In:
  Krishnapuram, B., Shah, M., Smola, A.J., Aggarwal, C.C., Shen, D., Rastogi,
  R. (eds.) Proceedings of the 22nd {ACM} {SIGKDD} International Conference on
  Knowledge Discovery and Data Mining, San Francisco, CA, USA, August 13-17,
  2016. pp. 855--864. {ACM} (2016)

\bibitem{DBLP:conf/icfca/HanikaH19}
Hanika, T., Hirth, J.: Conexp-clj - {A} research tool for {FCA}. In: Cristea,
  D., Ber, F.L., Missaoui, R., Kwuida, L., Sertkaya, B. (eds.) Supplementary
  Proceedings of {ICFCA} 2019 Conference and Workshops, Frankfurt, Germany,
  June 25-28, 2019. {CEUR} Workshop Proceedings, vol.~2378, pp. 70--75.
  CEUR-WS.org (2019)

\bibitem{DBLP:conf/icfca/Hanika0S19}
Hanika, T., Marx, M., Stumme, G.: Discovering implicational knowledge in
  wikidata. In: Cristea, D., Ber, F.L., Sertkaya, B. (eds.) Formal Concept
  Analysis - 15th International Conference, {ICFCA} 2019, Frankfurt, Germany,
  June 25-28, 2019, Proceedings. LNCS, vol. 11511, pp. 315--323. Springer
  (2019)

\bibitem{DBLP:conf/semweb/HoSGKW18}
Ho, V.T., Stepanova, D., Gad{-}Elrab, M.H., Kharlamov, E., Weikum, G.: Rule
  learning from knowledge graphs guided by embedding models. In: Vrandecic, D.,
  Bontcheva, K., Su{\'{a}}rez{-}Figueroa, M.C., Presutti, V., Celino, I.,
  Sabou, M., Kaffee, L., Simperl, E. (eds.) The Semantic Web - {ISWC} 2018 -
  17th International Semantic Web Conference, Monterey, CA, USA, October 8-12,
  2018, Proceedings, Part {I}. LNCS, vol. 11136, pp. 72--90. Springer (2018)

\bibitem{lecun2015learning}
LeCun, Y., Bengio, Y., Hinton, G.: Deep learning. Nature  \textbf{521},  436--
  (May 2015)

\bibitem{journals/corr/abs-1301-3781}
Mikolov, T., Chen, K., Corrado, G., Dean, J.: Efficient estimation of word
  representations in vector space. In: Bengio, Y., LeCun, Y. (eds.) ICLR
  (Workshop Poster) (2013)

\bibitem{mikolov2013distributed}
Mikolov, T., Sutskever, I., Chen, K., Corrado, G.S., Dean, J.: Distributed
  representations of words and phrases and their compositionality. In: Burges,
  C.J.C., Bottou, L., Ghahramani, Z., Weinberger, K.Q. (eds.) Advances in
  Neural Information Processing Systems 26: 27th Annual Conference on Neural
  Information Processing Systems 2013. Proceedings of a meeting held December
  5-8, 2013, Lake Tahoe, Nevada, United States. pp. 3111--3119 (2013)

\bibitem{mini08}
Mnih, A., Hinton, G.E.: A scalable hierarchical distributed language model. In:
  Koller, D., Schuurmans, D., Bengio, Y., Bottou, L. (eds.) Advances in Neural
  Information Processing Systems 21, Proceedings of the Twenty-Second Annual
  Conference on Neural Information Processing Systems, Vancouver, British
  Columbia, Canada, December 8-11, 2008. pp. 1081--1088. Curran Associates,
  Inc. (2008)

\bibitem{DBLP:journals/corr/abs-1710-04099}
Nielsen, F.{\AA}.: Wembedder: Wikidata entity embedding web service. CoRR
  \textbf{abs/1710.04099} (2017)

\bibitem{scikit-learn}
Pedregosa, F., Varoquaux, G., Gramfort, A., Michel, V., Thirion, B., Grisel,
  O., Blondel, M., Prettenhofer, P., Weiss, R., Dubourg, V., Vanderplas, J.,
  Passos, A., Cournapeau, D., Brucher, M., Perrot, M., Duchesnay, E.:
  Scikit-learn: Machine learning in {P}ython. Journal of Machine Learning
  Research  \textbf{12},  2825--2830 (2011)

\bibitem{AAAI1714732}
Peng, H., Li, J., Song, Y., Liu, Y.: Incrementally learning the hierarchical
  softmax function for neural language models. In: Singh, S.P., Markovitch, S.
  (eds.) Proceedings of the Thirty-First {AAAI} Conference on Artificial
  Intelligence, February 4-9, 2017, San Francisco, California, {USA}. pp.
  3267--3273. {AAAI} Press (2017)

\bibitem{perozzi14}
Perozzi, B., Al{-}Rfou, R., Skiena, S.: Deepwalk: online learning of social
  representations. In: Macskassy, S.A., Perlich, C., Leskovec, J., Wang, W.,
  Ghani, R. (eds.) The 20th {ACM} {SIGKDD} International Conference on
  Knowledge Discovery and Data Mining, {KDD} '14, New York, NY, {USA} - August
  24 - 27, 2014. pp. 701--710. {ACM} (2014)

\bibitem{DBLP:journals/semweb/RistoskiRNLP19}
Ristoski, P., Rosati, J., Noia, T.D., Leone, R.D., Paulheim, H.: Rdf2vec: {RDF}
  graph embeddings and their applications. Semantic Web  \textbf{10}(4),
  721--752 (2019)

\bibitem{rong14}
Rong, X.: word2vec parameter learning explained. CoRR  \textbf{abs/1411.2738}
  (2014)

\bibitem{conf/Rudolph07}
Rudolph, S.: Using {FCA} for encoding closure operators into neural networks.
  In: Priss, U., Polovina, S., Hill, R. (eds.) Conceptual Structures: Knowledge
  Architectures for Smart Applications, 15th International Conference on
  Conceptual Structures, {ICCS} 2007, Sheffield, UK, July 22-27, 2007,
  Proceedings. LNCS, vol.~4604, pp. 321--332. Springer (2007)

\bibitem{schlimmer1981mushroom}
Schlimmer, J.: Mushroom records drawn from the audubon society field guide to
  north american mushrooms. GH Lincoff (Pres), New York  (1981)

\bibitem{SCOTT1964233}
Scott, D.: Measurement structures and linear inequalities. Journal of
  Mathematical Psychology  \textbf{1}(2),  233 -- 247 (1964)

\bibitem{DBLP:journals/cacm/VrandecicK14}
Vrandecic, D., Kr{\"{o}}tzsch, M.: Wikidata: a free collaborative
  knowledgebase. Commun. {ACM}  \textbf{57}(10),  78--85 (2014)

\bibitem{DBLP:conf/aaai/WangZFC14}
Wang, Z., Zhang, J., Feng, J., Chen, Z.: Knowledge graph embedding by
  translating on hyperplanes. In: Brodley, C.E., Stone, P. (eds.) Proceedings
  of the Twenty-Eighth {AAAI} Conference on Artificial Intelligence, July 27
  -31, 2014, Qu{\'{e}}bec City, Qu{\'{e}}bec, Canada. pp. 1112--1119. {AAAI}
  Press (2014)

\bibitem{utaRepresentation}
Wille, U.: Representation of finite ordinal data in real vector spaces. In:
  Bock, H.H., Polasek, W. (eds.) Data Analysis and Information Systems. pp.
  228--240. Springer Berlin Heidelberg, Berlin, Heidelberg (1996)

\bibitem{wille1997role}
Wille, U.: The role of synthetic geometry in representational measurement
  theory. journal of mathematical psychology  \textbf{41}(1),  71--78 (1997)

\end{thebibliography}

\end{document}